\documentclass[oneside,reqno,american]{amsart}
\usepackage[T1]{fontenc}
\usepackage[utf8]{inputenc}
\usepackage{xcolor}
\usepackage{babel}
\usepackage{prettyref}
\usepackage{mathtools}
\usepackage{amstext}
\usepackage{amsthm}
\usepackage{amssymb}
\usepackage{wasysym}
\usepackage[pdfusetitle,
 bookmarks=true,bookmarksnumbered=false,bookmarksopen=false,
 breaklinks=false,pdfborder={0 0 0},pdfborderstyle={},backref=false,colorlinks=false]
 {hyperref}
\hypersetup{
 colorlinks=true,citecolor=blue,linkcolor=blue,linktocpage=true}

\makeatletter
\numberwithin{equation}{section}
\numberwithin{figure}{section}

\newrefformat{cor}{Corollary~\ref{#1}}
\newrefformat{subsec}{Section~\ref{#1}}
\newrefformat{lem}{Lemma~\ref{#1}}
\newrefformat{thm}{Theorem~\ref{#1}}
\newrefformat{sec}{Section~\ref{#1}}
\newrefformat{chap}{Chapter~\ref{#1}}
\newrefformat{prop}{Proposition~\ref{#1}}
\newrefformat{exa}{Example~\ref{#1}}
\newrefformat{tab}{Table~\ref{#1}}
\newrefformat{rem}{Remark~\ref{#1}}
\newrefformat{def}{Definition~\ref{#1}}
\newrefformat{fig}{Figure~\ref{#1}}
\newrefformat{claim}{Claim~\ref{#1}}

\AtBeginDocument{
  
}

\makeatother

\theoremstyle{plain}
\newtheorem{thm}{\protect\theoremname}[section]
\newtheorem{lem}[thm]{\protect\lemmaname}
\theoremstyle{remark}
\newtheorem{rem}[thm]{\protect\remarkname}
\newrefformat{lem}{Lemma \ref{#1}}
\newrefformat{rem}{Remark \ref{#1}}
\newrefformat{sec}{Section~\ref{#1}}
\newrefformat{thm}{Theorem \ref{#1}}
\providecommand{\lemmaname}{Lemma}
\providecommand{\remarkname}{Remark}
\providecommand{\theoremname}{Theorem}

\begin{document}
\title{A Class of Random-Kernel Network Models}
\begin{abstract}
We introduce random-kernel networks, a multilayer extension of random
feature models where depth is created by deterministic kernel composition
and randomness enters only in the outermost layer. We prove that deeper
constructions can approximate certain functions with fewer Monte Carlo
samples than any shallow counterpart, establishing a depth separation
theorem in sample complexity. 
\end{abstract}

\author{James Tian}
\address{Mathematical Reviews, 535 W. William St, Suite 210, Ann Arbor, MI
48103, USA}
\email{jft@ams.org}
\maketitle

\section{Introduction}\label{sec:1}

Random feature methods have become a standard tool for scaling up
kernel methods to large datasets. Since the seminal work of Rahimi
and Recht \cite{10.5555/2981562.2981710,10.5555/2981780.2981944},
a variety of random feature constructions have been proposed, giving
efficient algorithms for supervised learning, regression, and classification
while maintaining approximation guarantees derived from reproducing
kernel Hilbert space (RKHS) theory. The key idea is to approximate
a kernel $K(x,x')$ by a Monte Carlo average of randomized nonlinear
features of the input, thereby replacing expensive kernel evaluations
by finite-dimensional linear models.

Most existing analyses and applications of random features have focused
on shallow models, where the function is expressed as a linear combination
of random nonlinear features of the input. In parallel, the success
of deep neural networks has highlighted the role of compositional
structure and depth in achieving expressive and efficient representations.
This has motivated a growing line of work on deep kernel learning
\cite{pmlr-v51-wilson16,MR3714245}, the neural tangent kernel (NTK)
\cite{10.5555/3327757.3327948}, and other compositional kernel constructions
\cite{10.5555/2984093.2984132,NIPS2016_abea47ba,MR4400901}. These
approaches show that depth can fundamentally alter the geometry of
the induced RKHS and thereby improve learning capacity.

In this paper we develop and analyze random-kernel networks (RKNs),
a multilayer generalization of random feature models. Unlike classical
neural networks, where parameters are trained in every layer, our
construction fixes a base kernel and generates deeper layers by deterministic
composition. Each layer is defined as an RKHS induced by an integral
transform of the previous kernel, while randomness enters only through
Monte Carlo sampling in the outermost layer. Thus, in contrast to
standard deep networks, the only trainable parameters are the outermost
linear coefficients. Depth in our model refers to the compositional
structure of the kernel, not to additional layers of optimization.

This viewpoint is closely related to the literature on deep kernel
machines \cite{10.5555/2984093.2984132,pmlr-v51-wilson16}, randomized
neural networks \cite{471375,HUANG2006489}, and compositional kernel
learning \cite{NIPS2016_abea47ba,pmlr-v31-damianou13a}. However,
our contribution differs in that we provide a systematic definition
of multilayer random-kernel networks and a rigorous analysis of their
Monte Carlo approximation properties. In particular, we show that
depth can provide quantitative efficiency gains: there exist functions
that can be represented compactly at depth two, but for which every
depth-one approximation requires asymptotically more Monte Carlo samples.

Formally, we establish a depth separation theorem in sample complexity:
letting $N_{\mathrm{deep}}(\epsilon)$ and $N_{\mathrm{shallow}}(\epsilon)$
denote the minimal number of Monte Carlo samples needed by two-layer
and one-layer random-kernel networks, respectively, to achieve squared
$L^{2}$ error at most $\epsilon^{2}$, we prove that there exists
a depth-two function $f_{2}$ such that

\[
\frac{N_{\mathrm{shallow}}(\epsilon)}{N_{\mathrm{deep}}(\epsilon)}\to\infty
\]
as a construction parameter grows. This gives a strict, tunable efficiency
gain for depth. To our knowledge, this appears to be the first result
that establishes depth separation for Monte Carlo sample complexity
in kernel-based random feature models.

The present setting provides a bridge between random feature methods
and deep kernel constructions. On the one hand, it inherits the algorithmic
simplicity of random features, since training reduces to linear regression
in the outermost coefficients. On the other hand, it exhibits structural
advantages of depth traditionally associated with neural networks.
This makes random-kernel networks a promising mathematical model for
understanding the role of depth beyond standard parametric optimization.

\textbf{Organization.} \prettyref{sec:2} defines random-kernel networks
and their associated RKHSs. \prettyref{sec:3} develops the special
case of the ReLU kernel, deriving explicit formulas and a Monte Carlo
approximation bound for last-layer sampling in this widely studied
activation. \prettyref{sec:4} presents a general depth separation
theorem, extending the ReLU analysis to a wide class of base kernels.

\section{Compositional Kernels via Random Features}\label{sec:2}

We now introduce the general framework of random-kernel networks,
which will serve as the foundation for the ReLU case study in \prettyref{sec:3}.
The construction starts from a base kernel-generating function and
builds deeper layers recursively by composition. Depth in this model
reflects the compositional structure of the kernels, rather than additional
layers of trainable parameters.

\textbf{Settings.} Let $X\subset\mathbb{R}^{d}$ be the input domain.
Fix a measurable, bounded function $K\colon\mathbb{R}\times\mathbb{R}\rightarrow\mathbb{R}$.
Let 
\[
Z=\mathbb{R}^{d}\times\mathbb{R}\times\mathbb{R}
\]
with generic element $z=\left(a,b,t\right)$, and let $\rho$ be a
probability measure on $Z$. 

Define the first-layer atom and kernel by 
\begin{align*}
\psi^{\left(1\right)}_{z}\left(x\right) & =K\left(\left\langle a,x\right\rangle +b,t\right)\\
K^{\left(1\right)}\left(x,x'\right) & =\int_{Z}\psi^{\left(1\right)}_{z}\left(x\right)\psi^{\left(1\right)}_{z}\left(x'\right)d\rho\left(z\right),\quad x,x'\in X.
\end{align*}
Then $K^{\left(1\right)}$ is positive definite (as an inner product
in $L^{2}\left(\rho\right)$) and determines an RKHS $\mathcal{H}^{\left(1\right)}$. 

For deeper layers we need to evaluate $K^{\left(l-1\right)}\left(a,x\right)$
for arbitrary $a\in\mathbb{R}^{d}$. We therefore extend $K^{\left(1\right)}$
to $\mathbb{R}^{d}\times X$ by the same integral formula. Recursively,
once $K^{\left(l-1\right)}$ is defined and extended to $\mathbb{R}^{d}\times X$,
set for $l\geq2$, 
\begin{align*}
\psi^{\left(l\right)}_{z}\left(x\right) & =K\left(K^{\left(l-1\right)}\left(a,x\right)+b,t\right),\\
K^{\left(l\right)}\left(x,x'\right) & =\int_{Z}\psi^{\left(l\right)}_{z}\left(x\right)\psi^{\left(l\right)}_{z}\left(x'\right)d\rho\left(z\right),\quad x,x'\in X.
\end{align*}
Again, $K^{\left(l\right)}$ is positive definite and determines an
RKHS $\mathcal{H}^{\left(l\right)}$. 

Each $\mathcal{H}^{\left(i\right)}$, $i=1,\dots,l$, consists of
functions of the form (see e.g., \cite{MR51437} and \cite[Theorem 4.21]{MR2450103})
\begin{equation}
f\left(x\right)=\int_{Z}c\left(z\right)\psi^{\left(i\right)}_{z}\left(x\right)d\rho\left(x\right),\quad c\in L^{2}\left(\rho\right),\label{eq:b1}
\end{equation}
with norm 
\begin{equation}
\left\Vert f\right\Vert _{\mathcal{H}^{\left(i\right)}}=\inf\left\{ \left\Vert c\right\Vert _{L^{2}\left(\rho\right)}:f=\int c\psi^{\left(i\right)}_{z}d\rho\right\} .\label{eq:b2}
\end{equation}

A depth-$l$ random-kernel network is defined as follows: sample i.i.d.
parameters $z_{j}=\left(a_{j},b_{j},t_{j}\right)\sim\rho$ and set
\[
f_{N}\left(x\right)=\sum^{N}_{j=1}\alpha_{j}\psi^{\left(l\right)}_{z_{j}}\left(x\right),\quad x\in X.
\]

In this model, only the coefficients $\alpha_{j}$ in the outermost
layer are trained. The inner layers appear through deterministic integral
transforms that define the kernels $K^{\left(l\right)}$. This is
therefore different from the traditional notion of a multilayer neural
network, where weights in every layer are optimized. Here the notion
of depth refers to the compositional structure of the kernel rather
than to additional layers of trainable parameters. The benefit of
this construction is that certain compositional target functions can
be represented with substantially smaller outermost coefficient norms
(see $\left\Vert c\right\Vert _{L^{2}\left(\rho\right)}$ in \eqref{eq:b1}--\eqref{eq:b2})
than is possible in the single-layer case, leading to provable efficiency
gains in the number of random features required for approximation. 

\section{Depth Separation on a Slice: A Two-Layer ReLU Example}\label{sec:3}

This section treats a concrete two-layer ReLU case to illustrate the
depth-efficiency phenomenon on a one-dimensional slice. The next section
states and proves a general depth-$l$ theorem under broader assumptions. 

Our strategy is to design a specific two-layer function that is inherently
compositional and then analyze what is required to approximate it
using both a deep (two-layer) and a shallow (one-layer) random-kernel
network. To make the analysis tractable, we focus on the approximation
error along a specific one-dimensional slice of the input space, $x=sv$.
This is sufficient to establish the desired separation; if a shallow
model requires a large number of samples to approximate the function
on this single line, its sample complexity for the function over the
entire domain must also be large.

The argument proceeds in three main steps. First, we construct a ``steep''
two-layer target function $f_{2}$, whose steepness is controlled
by a tunable parameter $m_{*}$. Second, we show that this function
can be represented efficiently in the two-layer space $\mathcal{H}^{(2)}$
with a small and well-behaved coefficient norm. Finally, we prove
that any function in the one-layer space $\mathcal{H}^{(1)}$ that
approximates $f_{2}$ is forced to have a large norm that scales with
$m_{*}$. This occurs because shallow functions have a bounded derivative
tied to their norm, and approximating the steep target function forces
this norm to explode. Comparing the Monte Carlo sample complexities,
which are proportional to the squared norms, reveals a separation
that grows with $m^{2}_{*}$, thus proving the main assertion (\prettyref{thm:7}). 

Consider $X=\mathbb{R}^{d}$, $K\left(s,t\right)=\sigma\left(s\right)=\max\left\{ 0,s\right\} =$
ReLU, independent of $t$. Fix a Borel probability measure $\rho$
on the parameter space $Z=\mathbb{R}^{d}\times\mathbb{R}$ with generic
elements $z=\left(a,b\right)$, such that 
\begin{equation}
\int_{Z}\left(\left\Vert a\right\Vert ^{2}+b^{2}\right)d\rho\left(a,b\right)<\infty.\label{eq:c1}
\end{equation}
Define the first layer atom and kernel by 
\[
\psi^{\left(1\right)}_{z}\left(x\right)=\sigma\left(\left\langle a,x\right\rangle +b\right),\quad K^{\left(1\right)}\left(x,x'\right)=\int_{Z}\psi^{\left(1\right)}_{z}\left(x\right)\psi^{\left(1\right)}_{z}\left(x'\right)d\rho\left(z\right).
\]

\begin{lem}
\label{lem:1}Under \eqref{eq:c1}, $K^{\left(1\right)}\left(x,x'\right)$
is well defined for all $x,x'\in\mathbb{R}^{d}$.
\end{lem}

\begin{proof}
For any $x\in\mathbb{R}^{d}$, 
\[
\left|\psi_{z}\left(x\right)\right|=\sigma\left(\left\langle a,x\right\rangle +b\right)\leq\left|\left\langle a,x\right\rangle \right|+\left|b\right|\leq\left\Vert a\right\Vert \left\Vert x\right\Vert +\left|b\right|
\]
so that 
\[
\int_{Z}\left|\psi^{\left(1\right)}_{z}\left(x\right)\right|^{2}d\rho\left(z\right)\leq\int_{Z}\left(\left\Vert a\right\Vert \left\Vert x\right\Vert +\left|b\right|\right)^{2}d\rho\left(z\right)<\infty
\]
by \eqref{eq:c1}. By Cauchy-Schwarz, 
\[
\left|K^{\left(1\right)}\left(x,x'\right)\right|\leq\left(\int_{Z}\left|\psi^{\left(1\right)}_{z}\left(x\right)\right|^{2}d\rho\left(z\right)\right)^{1/2}\left(\int_{Z}\left|\psi^{\left(1\right)}_{z}\left(x'\right)\right|^{2}d\rho\left(z\right)\right)^{1/2}<\infty
\]
for all $x,x'\in\mathbb{R}^{d}$.
\end{proof}
With this, the second layer is also well defined: 
\[
\psi^{\left(2\right)}_{z}\left(x\right)=\sigma\left(K^{\left(1\right)}\left(a,x\right)+b\right),\quad K^{\left(2\right)}\left(x,x'\right)=\int_{Z}\psi^{\left(2\right)}_{z}\left(x\right)\psi^{\left(2\right)}_{z}\left(x'\right)d\rho\left(z\right)
\]
because $K^{\left(1\right)}$ is finite, so the exact same square-integrability
and Cauchy-Schwarz argument applies again.

Now, fix a unit vector $v\in\mathbb{R}^{d}$. Let
\[
\rho=\frac{1}{2}\rho_{uni}+\frac{1}{2}\rho_{spec}
\]
 where $\rho_{uni}$ is any probability measure with finite second
moments (i.e., \eqref{eq:c1} holds), and 
\[
\rho_{spec}=\frac{1}{4}\delta_{\left(v,B\right)}+\frac{1}{4}\left(\delta_{\left(v,-\beta_{0}\right)}+\delta_{\left(v,-\beta_{1/2}\right)}+\delta_{\left(v,-\beta_{1}\right)}\right)
\]
with $\beta_{0},\beta_{1/2},\beta_{1}>1$ (precise values will be
set later). 

Define 
\[
G\left(s\right)\coloneqq K^{\left(1\right)}\left(v,sv\right)=\int_{Z}\sigma\left(\left\langle a,v\right\rangle +b\right)\sigma\left(\left\langle a,sv\right\rangle +b\right)d\rho\left(a,b\right),\quad s\in\left[0,1\right].
\]

\begin{lem}
\label{lem:3}There exists a constant $m_{*}>0$ such that 
\[
G'\left(s\right)\geq m_{*}
\]
for a.e. $s\in\left[0,1\right]$. 
\end{lem}

\begin{proof}
By dominated convergence / differentiating under integral, we have
\[
G'\left(s\right)=\int_{Z}\sigma\left(\left\langle a,v\right\rangle +b\right)1_{\left\{ \left\langle a,sv\right\rangle +b>0\right\} }\left\langle a,v\right\rangle d\rho\left(a,b\right)\quad a.e.\:s.
\]
Splitting $\rho=\frac{1}{2}\rho_{uni}+\frac{1}{2}\rho_{spec}$, then
\[
G'\left(s\right)=I_{uni}+I_{spec}
\]
where
\begin{align*}
I_{uni} & =\frac{1}{2}\int_{Z}\sigma\left(\left\langle a,v\right\rangle +b\right)1_{\left\{ \left\langle a,sv\right\rangle +b>0\right\} }\left\langle a,v\right\rangle d\rho_{uni}\left(a,b\right)\\
 & \geq-\frac{1}{2}\int_{Z}\sigma\left(\left\langle a,v\right\rangle +b\right)\left|\left\langle a,v\right\rangle \right|d\rho_{uni}\left(a,b\right)\eqqcolon-M
\end{align*}
and 
\[
I_{spec}=\frac{1}{2}\int_{Z}\sigma\left(\left\langle a,v\right\rangle +b\right)1_{\left\{ \left\langle a,sv\right\rangle +b>0\right\} }\left\langle a,v\right\rangle d\rho_{spec}\left(a,b\right)=\frac{1}{8}\left(1+B\right).
\]
Note that the contribution of $\frac{1}{4}\left(\delta_{\left(v,-\beta_{0}\right)}+\delta_{\left(v,-\beta_{1/2}\right)}+\delta_{\left(v,-\beta_{1}\right)}\right)$
in $I_{spec}$ is zero, since $\left\langle v,sv\right\rangle -\beta<0$,
for any $\beta\in\left\{ \beta_{0},\beta_{1/2},\beta_{1}\right\} $.
Thus
\[
G'\left(s\right)\geq\frac{1}{8}\left(1+B\right)-M\eqqcolon m_{*}
\]
for a.e. $s\in\left[0,1\right]$, which is positive for sufficiently
large $B$. 
\end{proof}
Let $\triangle=G\left(1\right)-G\left(0\right)$, and 
\[
\beta_{0}=G\left(0\right),\quad\beta_{1}=G\left(1\right),\quad\beta_{1/2}=\frac{1}{2}\left(G\left(0\right)+G\left(1\right)\right).
\]
Let $T$ be the tent function on $\left[\beta_{0},\beta_{1}\right]$:
\[
T\left(s\right)=\frac{2}{\triangle}\sigma\left(s-\beta_{0}\right)-\frac{4}{\triangle}\sigma\left(s-\beta_{1/2}\right)+\frac{2}{\triangle}\sigma\left(s-\beta_{1}\right).
\]
From above, $G\left(s\right)=K^{\left(1\right)}\left(v,sv\right)$,
$s\in\left[0,1\right]$, is increasing\@. Notice that 
\begin{align*}
G\left(0\right) & =K^{\left(1\right)}\left(v,0\right)=\int_{Z}\sigma\left(\left\langle a,v\right\rangle +b\right)\sigma\left(b\right)d\rho\left(a,b\right)\\
 & \geq\frac{1}{2}\int_{Z}\sigma\left(\left\langle a,v\right\rangle +b\right)\sigma\left(b\right)d\rho_{spec}\left(a,b\right)\\
 & =\frac{1}{8}\left(1+B\right)B
\end{align*}
so we can always choose $G\left(0\right)>1$. Thus, $\beta_{0},\beta_{1/2},\beta_{1}>1$. 

Now define 
\[
f_{2}\left(x\right)=T\left(K^{\left(1\right)}\left(v,x\right)\right).
\]

\begin{lem}
$f_{2}\in\mathcal{H}^{\left(2\right)}$ with $\left\Vert f_{2}\right\Vert _{\mathcal{H}^{\left(2\right)}}\leq\frac{8\sqrt{3}}{\triangle}$. 
\end{lem}

\begin{proof}
Choose the second layer coefficient $c\in L^{2}\left(\rho\right)$
to be supported on three points: 
\[
c\left(z_{0}\right)=\frac{16}{\triangle},\quad c\left(z_{1/2}\right)=-\frac{32}{\triangle},\quad c\left(z_{1}\right)=\frac{16}{\triangle},
\]
and set $c\equiv0$ elsewhere. Here, $z_{0}=\left(v,-\beta_{0}\right)$,
$z_{1/2}=\left(v,-\beta_{1/2}\right)$, and $z_{1}=\left(v,-\beta_{1}\right)$.

Then, for every $x\in\mathbb{R}^{d}$, 
\begin{eqnarray*}
 &  & \int_{Z}c\left(z\right)\psi^{\left(2\right)}_{z}\left(x\right)d\rho\left(z\right)\\
 & = & \frac{1}{2}\int_{Z}c\left(z\right)\psi^{\left(2\right)}_{z}\left(x\right)d\rho_{spec}\left(z\right)\\
 & = & \frac{2}{\triangle}\sigma\left(K^{\left(1\right)}\left(v,x\right)-\beta_{0}\right)-\frac{4}{\triangle}\sigma\left(K^{\left(1\right)}\left(v,x\right)-\beta_{1/2}\right)+\frac{2}{\triangle}\sigma\left(K^{\left(1\right)}\left(v,x\right)-\beta_{1}\right)\\
 & = & f_{2}\left(x\right).
\end{eqnarray*}
Therefore, $f_{2}\in\mathcal{H}^{\left(2\right)}$, and
\begin{align*}
\left\Vert f_{2}\right\Vert _{\mathcal{H}^{\left(2\right)}} & \leq\left\Vert c\right\Vert _{L^{2}\left(\rho\right)}\\
 & =\left(\frac{1}{8}\left|c\left(z_{0}\right)\right|^{2}+\frac{1}{8}\left|c\left(z_{1/2}\right)\right|^{2}+\frac{1}{8}\left|c\left(z_{1}\right)\right|^{2}\right)^{1/2}=\frac{8\sqrt{3}}{\triangle}.
\end{align*}
(See the definition of $\left\Vert f\right\Vert _{\mathcal{H}^{\left(i\right)}}$
in \eqref{eq:b1}--\eqref{eq:b2}.) 
\end{proof}
\begin{lem}
For any $g\in\mathcal{H}^{\left(1\right)}$, set $q\left(s\right)=g\left(sv\right)$,
$s\in\left[0,1\right]$. Then 
\[
\left|q'\left(s\right)\right|\leq\sigma_{v}\left\Vert g\right\Vert _{\mathcal{H}^{\left(1\right)}}\quad a.e.\:s
\]
for some constant $\sigma_{v}$. 
\end{lem}

\begin{proof}
Take any $g\left(x\right)=\int_{Z}c\left(a,b\right)\sigma\left(\left\langle a,x\right\rangle +b\right)d\rho\left(a,b\right)\in\mathcal{H}^{\left(1\right)}$.
Then, 
\[
q'\left(s\right)=\int_{Z}c\left(a,b\right)1_{\left\{ \left\langle a,sv\right\rangle +b>0\right\} }\left\langle a,v\right\rangle d\rho\left(a,b\right).
\]
By Cauchy-Schwarz, 
\[
\left|q'\left(s\right)\right|\leq\underset{=\sigma_{v}}{\underbrace{\left(\int_{Z}\left|\left\langle a,v\right\rangle \right|^{2}d\rho\left(a,b\right)\right)^{1/2}}}\left\Vert c\right\Vert _{L^{2}\left(\rho\right)}\quad a.e.\;s.
\]
Since $\left\Vert g\right\Vert _{\mathcal{H}^{\left(1\right)}}$ is
the infimum of $\left\Vert c\right\Vert _{L^{2}\left(\rho\right)}$
over all such representations, we have 
\[
\left|q'\left(s\right)\right|\leq\sigma_{v}\left\Vert g\right\Vert _{\mathcal{H}^{\left(1\right)}}\quad a.e.\:s\in\left[0,1\right].
\]
\end{proof}
\begin{lem}
Let $h\left(s\right)=f_{2}\left(sv\right)$, $s\in\left[0,1\right]$.
Then, 
\[
\left|h'\left(s\right)\right|\geq\frac{2m_{*}}{\triangle}>0\quad a.e.
\]
where $m_{*}$ is the lower bound of $G'$ from \prettyref{lem:3}. 
\end{lem}

\begin{proof}
Let
\[
h\left(s\right)\coloneqq f_{2}\left(sv\right)=T\left(K^{\left(1\right)}\left(v,sv\right)\right)=T\left(G\left(s\right)\right),\quad s\in\left[0,1\right].
\]
Since $G$ is increasing, there exists a unique $s_{*}$ where $G\left(s_{*}\right)=\beta_{1/2}$,
and 
\[
h'\left(s\right)=\begin{cases}
\frac{2}{\triangle}G'\left(s\right) & s<s_{*}\\
-\frac{2}{\triangle}G'\left(s\right) & s>s_{*}\\
0 & \text{elsewhere}
\end{cases}\quad a.e.
\]
Using $G'\left(s\right)\geq m_{*}$ a.e., we get that 
\[
\left|h'\left(s\right)\right|\geq\frac{2m_{*}}{\triangle}>0
\]
a.e. on $\left[0,1\right]$.
\end{proof}
\begin{lem}
\label{lem:7}For any $g\in\mathcal{H}^{\left(1\right)}$, set $q\left(s\right)=g\left(sv\right)$,
$s\in\left[0,1\right]$. Let $h\left(s\right)=f_{2}\left(sv\right)$,
$s\in\left[0,1\right]$. If 
\[
\int^{1}_{0}\left|q\left(s\right)-h\left(s\right)\right|^{2}ds\leq\epsilon^{2}
\]
then 
\[
\left\Vert g\right\Vert _{\mathcal{H}^{\left(1\right)}}\geq\frac{1}{\sigma_{v}}\left(\frac{2m_{*}}{\triangle}-\sqrt{48}\epsilon\right)_{+}.
\]
\end{lem}

\begin{proof}
On the interval $\left(0,s_{*}\right)$, the target function $h$
has slope $\geq\frac{2m_{*}}{\triangle}$, and $q$ has slope bounded
by $\left|q'\left(s\right)\right|\leq\sigma_{v}\left\Vert g\right\Vert _{\mathcal{H}^{\left(1\right)}}$
a.e. So the difference in slopes between $h$ and $q$ is at least
$\frac{2m_{*}}{\triangle}-\sigma_{v}\left\Vert g\right\Vert _{\mathcal{H}^{\left(1\right)}}$.
A standard one-dimensional calculus estimate (minimizing over additive
constants to allow vertical alignment) yields
\[
\int^{s_{*}}_{0}\left|q-h\right|^{2}\geq\left(\frac{2m_{*}}{\triangle}-\sigma_{v}\left\Vert g\right\Vert _{\mathcal{H}^{\left(1\right)}}\right)^{2}\frac{1}{12}\left|s_{*}\right|^{3}
\]
The same bound holds on $\left[s_{*},1\right]$, so 
\begin{align*}
\int^{1}_{0}\left|q-h\right|^{2} & \geq\frac{1}{12}\left(\frac{2m_{*}}{\triangle}-\sigma_{v}\left\Vert g\right\Vert _{\mathcal{H}^{\left(1\right)}}\right)^{2}\left(\left|s_{*}\right|^{3}+\left|1-s_{*}\right|^{3}\right)\\
 & \geq\frac{1}{48}\left(\frac{2m_{*}}{\triangle}-\sigma_{v}\left\Vert g\right\Vert _{\mathcal{H}^{\left(1\right)}}\right)^{2}.
\end{align*}

Setting $\int^{1}_{0}\left|q-h\right|^{2}\leq\epsilon^{2}$ and solving
for $\left\Vert g\right\Vert _{\mathcal{H}^{\left(1\right)}}$ gives
the assertion. 
\end{proof}
\begin{lem}
\label{lem:c7}Consider approximation error measured on the line $x=sv$,
$s\in\left[0,1\right]$, with $\mu_{v}$ the uniform measure on $\left[0,1\right]$.
For any 
\[
f\left(x\right)=\int_{Z}c\left(z\right)\psi^{\left(l\right)}_{z}\left(x\right)d\rho\left(z\right)\in\mathcal{H}^{\left(l\right)},
\]
draw $z_{1},\dots,z_{N}\overset{i.i.d}{\sim}\rho$ and set the Monte
Carlo (MC) estimator 
\[
\hat{f}_{N}\left(x\right)=\frac{1}{N}\sum^{N}_{j=1}c\left(z_{j}\right)\psi^{\left(l\right)}_{z_{j}}\left(x\right).
\]
Then $\mathbb{E}[\hat{f}_{N}]=f$ and 
\[
\mathbb{E}\left[\left\Vert \hat{f}_{N}-f\right\Vert ^{2}_{L^{2}\left(\mu_{v}\right)}\right]\leq\frac{V_{l}}{N}\left\Vert c\right\Vert ^{2}_{L^{2}\left(\rho\right)},\quad V_{l}\coloneqq\sup_{z}\int^{1}_{0}\psi^{\left(l\right)}_{z}\left(sv\right)^{2}ds<\infty.
\]
\end{lem}

\begin{proof}
This is a variant of Theorem 9 in \cite{tian2025ridgekernelaveraginguniform}.
The finiteness of $V_{l}$ follows from the moment assumption \eqref{eq:c1}.
Details are omitted here. 
\end{proof}
Applying \prettyref{lem:c7} to our two cases on the line $\left[0,1\right]$,
we get:
\begin{thm}[Depth separation in MC sample complexity]
\label{thm:7}Let the approximation error be measured on the slice
$x=sv$, $s\in\left[0,1\right]$, with respect to the uniform measure
$\mu_{v}$. Given $\epsilon>0$, let 
\begin{itemize}
\item $N_{deep}\left(\epsilon\right)$: the minimal number of MC samples
in the last layer of a two-layer representation required to achieve
squared $L^{2}\left(\mu_{v}\right)$ error $\leq\epsilon^{2}$.
\item $N_{shallow}\left(\epsilon\right)$: the analogous minimal number
of MC samples when restricted to one-layer representations.
\end{itemize}
Then there exists a function $f_{2}\in\mathcal{H}^{\left(2\right)}$,
such that, for sufficiently small $\epsilon$, 
\[
\frac{N_{shallow}\left(\epsilon\right)}{N_{deep}\left(\epsilon\right)}\sim m^{2}_{*}\rightarrow\infty
\]
where $m_{*}$ is a tunable construction parameter. In particular,
this yields a quantitatively controllable efficiency gain for two-layer
over one-layer representations. 

\end{thm}

\begin{proof}
From the previous step we built $f_{2}$ with coefficient $c$ supported
on three atoms and found 
\[
\left\Vert f_{2}\right\Vert _{\mathcal{H}^{\left(2\right)}}\leq\left\Vert c\right\Vert _{L^{2}\left(\rho\right)}=\frac{8\sqrt{3}}{\triangle}.
\]
Therefore, for its last-layer MC approximation $\hat{f}_{N}$ at depth
2, 
\[
\mathbb{E}\left[\left\Vert \hat{f}_{N}-f_{2}\right\Vert ^{2}_{L^{2}\left(\mu_{v}\right)}\right]\leq\frac{V_{2}}{N}\left\Vert c\right\Vert ^{2}_{L^{2}\left(\rho\right)}\lesssim\frac{V_{2}}{N\triangle^{2}}.
\]
To reach error $\leq\epsilon^{2}$, it suffices that $N\gtrsim\frac{V_{2}}{\triangle^{2}\epsilon^{2}}$,
and therefore, 
\[
N_{\text{deep}}\left(\epsilon\right)\lesssim\frac{V_{2}}{\triangle^{2}\epsilon^{2}}.
\]

For any depth-1 candidate $g\in\mathcal{H}^{\left(1\right)}$, let
$\hat{g}_{N}$ be its MC estimator, then 
\begin{align*}
\mathbb{E}\left[\left\Vert \hat{g}_{N}-f_{2}\right\Vert ^{2}_{L^{2}\left(\mu_{v}\right)}\right] & =\mathbb{E}\left[\left\Vert \hat{g}_{N}-g\right\Vert ^{2}_{L^{2}\left(\mu_{v}\right)}\right]+\left\Vert g-f_{2}\right\Vert ^{2}_{L^{2}\left(\mu_{v}\right)}\\
 & \leq\frac{V_{1}}{N}\left\Vert g\right\Vert ^{2}_{\mathcal{H}^{1}}+\left\Vert g-f_{2}\right\Vert ^{2}_{L^{2}\left(\mu_{v}\right)}.
\end{align*}
Setting each term $\leq\epsilon^{2}/2$, \prettyref{lem:7} implies
\[
\left\Vert g\right\Vert ^{2}_{\mathcal{H}^{1}}\geq\left(\frac{1}{\sigma_{v}}\left(\frac{2m_{*}}{\triangle}-\sqrt{24}\epsilon\right)_{+}\right)^{2}.
\]
Plugging this into $\frac{V_{1}}{N}\left\Vert g\right\Vert ^{2}_{\mathcal{H}^{1}}\leq\epsilon^{2}/2$
gives the shallow sample-complexity lower bound
\[
N_{\text{shallow}}\left(\epsilon\right)\geq\frac{2V_{1}}{\epsilon^{2}}\left(\frac{1}{\sigma_{v}}\left(\frac{2m_{*}}{\triangle}-\sqrt{24}\epsilon\right)_{+}\right)^{2}.
\]
In particular, for sufficiently small $\epsilon$, 
\[
N_{\text{shallow}}\left(\epsilon\right)\apprge\frac{V_{1}}{\epsilon^{2}}\left(\frac{m_{*}}{\sigma_{v}\triangle}\right)^{2}.
\]
Thus, 
\[
\frac{N_{\text{shallow}}\left(\epsilon\right)}{N_{\text{deep}}\left(\epsilon\right)}\apprge\frac{V_{1}}{V_{2}}\frac{m^{2}_{*}}{\sigma^{2}_{v}}\sim m^{2}_{*}
\]
since $V_{1},V_{2},\sigma^{2}_{v}$ are fixed constants. Finally,
$m_{*}=\frac{1}{8}\left(1+B\right)-M\sim B$ and $B$ can be taken
to be arbitrary large, so the ratio can be arbitrarily large.

The key point is that the Monte Carlo variance bound 
\[
\mathbb{E}\left[\left\Vert \hat{f}_{N}-f\right\Vert ^{2}_{L^{2}\left(\mu_{v}\right)}\right]\leq\frac{V_{l}}{N}\left\Vert c\right\Vert ^{2}_{L^{2}(\rho)}
\]
always provides a sufficient sample size once a representation with
known coefficient norm is fixed.

In the deep case we explicitly construct $f_{2}\in\mathcal{H}^{(2)}$
using a single outer atom, so that $\left\Vert c\right\Vert _{L^{2}(\rho)}$
is known and bounded. This yields an upper bound on $N_{\text{deep}}\left(\epsilon\right)$.
In contrast, for shallow networks we cannot exhibit a good low-norm
representation of $f_{2}$. Instead, \prettyref{lem:7} and its consequences
force any shallow approximant $g$ to have a large RKHS norm, which
turns the same variance inequality into a necessary condition on $N$.
Thus we obtain a lower bound on $N_{\mathrm{shallow}}\left(\epsilon\right)$.
The depth separation arises precisely from this asymmetry: constructive
efficiency at depth-2 versus unavoidable norm inflation at depth-1.
\end{proof}
\begin{rem}[Comparison with classical ReLU networks]
The phenomenon established above is directly analogous to the well-known
depth separation results for ReLU networks. It has long been observed
that shallow ReLU networks are universal approximators, but may require
exponentially many hidden units to approximate certain functions that
deeper networks realize with polynomial size. For instance, Telgarsky
\cite{pmlr-v49-telgarsky16} constructed one-dimensional oscillatory
functions (iterated sawtooths) that can be represented exactly by
depth-$O(L)$ ReLU networks of constant width, but whose approximation
by any depth-2 ReLU network requires width exponential in $L$. Eldan
and Shamir \cite{68fe7375fd5f4ff7b454113df5dc9d97} provided a higher-dimensional
counterpart, showing that there exists a radial function in $\mathbb{R}^{d}$
realizable by a depth-3 ReLU network with polynomially many units,
but requiring super-polynomial width for any depth-2 network. In both
cases, depth does not enlarge the class of approximable functions
but yields dramatic improvements in the size/complexity of the representation.

Our construction shows that the same phenomenon arises in the random-kernel
model. Here the role of “width” is played by the number of random
features drawn in the outermost layer, or equivalently by the coefficient
norm $\|\alpha\|$ in the Monte Carlo approximation. A depth-$l$
kernel network can realize certain compositional functions with substantially
smaller outer-layer coefficients than any depth-1 kernel network,
leading to provable gains in the number of random features required.
Thus, while the precise efficiency metric differs from the parameter-counting
arguments in the ReLU literature, the underlying principle is the
same: depth enables exponential or polynomial savings by matching
the compositional structure of the target function.
\end{rem}

\section{Depth Separation for Random-Kernel Networks}\label{sec:4}

We now state and prove a general depth separation theorem for random-kernel
networks, showing that multi-layer constructions can be strictly more
efficient than single-layer ones in terms of Monte Carlo sample complexity.

Assumptions:
\begin{itemize}
\item[(A1)] $K\colon\mathbb{R}\times\mathbb{R}\rightarrow\mathbb{R}$ is measurable,
bounded, and Lipschitz in its first argument, uniformly in the second,
i.e., 
\begin{equation}
\left|K\left(s,t\right)-K\left(s',t\right)\right|\leq L_{K}\left|s-s'\right|\label{eq:d1}
\end{equation}
for all $s,s',t$, and for some constant $L_{K}$.
\item[(A2)] Parameter space $Z=\mathbb{R}^{d}\times\mathbb{R}\times\mathbb{R}$
with generic element $z=\left(a,b,t\right)$, and $\rho$ is a probability
measure on $Z$ such that 
\begin{equation}
\int_{Z}\left\Vert a\right\Vert ^{2}d\rho\left(a,b,t\right)<\infty.\label{eq:d2}
\end{equation}
(No further moment condition on $b,t$ is needed since $K$ is bounded.)
\item[(A3)] There exists $t_{+}\in\mathbb{R}$, an open interval $\left(u_{0},u_{1}\right)$
of $\mathbb{R}$, and constants $c_{1},c_{2}>0$ such that for a.e.
$u\in\left(u_{0},u_{1}\right)$, 
\[
K\left(u,t_{+}\right)\geq c_{1},\quad\partial_{1}K\left(u,t_{+}\right)\geq c_{2}.
\]
(That is, there is a short interval where $K$ is positive and has
a positive slope. This holds for many bounded Lipschitz generators.) 
\end{itemize}
Using these, the layered kernels $K^{\left(l\right)}$ and RKHSs $\mathcal{H}^{\left(l\right)}$
are well defined exactly as in \prettyref{lem:1} (boundedness of
$K$ makes all layer integrals finite).

Fix a unit vector $v\in\mathbb{R}^{d}$. Define the constant 
\begin{equation}
\sigma_{v}\coloneqq\left(\int_{Z}\left|\left\langle a,v\right\rangle \right|^{2}d\rho\left(a,b,t\right)\right)^{1/2}<\infty.\label{eq:d3}
\end{equation}

\begin{lem}
\label{lem:d-1}For any $g\in\mathcal{H}^{\left(1\right)}$ and $q\left(s\right)\coloneqq g\left(sv\right)$,
$s\in\left[0,1\right]$, one has 
\[
\left|q'\left(s\right)\right|\leq L_{K}\sigma_{v}\left\Vert g\right\Vert _{\mathcal{H}^{\left(1\right)}}\quad a.e.\:s.
\]
\end{lem}

\begin{proof}
By definition of $\mathcal{H}^{\left(1\right)}$ (see \eqref{eq:b1}--\eqref{eq:b2}),
for every $\eta>0$ there exists $c\in L^{2}\left(\rho\right)$ with
$\left\Vert c\right\Vert _{L^{2}\left(\rho\right)}\leq\left\Vert g\right\Vert _{\mathcal{H}^{\left(1\right)}}+\eta$
and 
\[
g\left(x\right)=\int_{Z}c\left(z\right)K\left(\left\langle a,x\right\rangle +b,t\right)d\rho\left(z\right).
\]
Hence along the slice $x=sv$, 
\[
q\left(s\right)=\int_{Z}c\left(z\right)K\left(\left\langle a,sv\right\rangle +b,t\right)d\rho\left(z\right).
\]
For any $s,s'$, using the Lipschitz assumption \eqref{eq:d1} and
Cauchy-Schwarz, 
\begin{align*}
\left|q\left(s\right)-q\left(s'\right)\right| & \leq L_{K}\left|s-s'\right|\int_{Z}\left|c\left(z\right)\right|\left|\left\langle a,v\right\rangle \right|d\rho\left(z\right)\\
 & \leq L_{K}\left|s-s'\right|\left(\int_{Z}\left|c\left(z\right)\right|^{2}d\rho\left(z\right)\right)^{1/2}\underset{=\sigma_{v}\:\eqref{eq:d3}}{\underbrace{\left(\int_{Z}\left|\left\langle a,v\right\rangle \right|^{2}d\rho\left(z\right)\right)^{1/2}}}\\
 & =L_{K}\left\Vert c\right\Vert _{L^{2}\left(\rho\right)}\sigma_{v}\left|s-s'\right|.
\end{align*}
Thus $q$ is Lipschitz with constant $\leq L_{K}\sigma_{v}\left\Vert c\right\Vert _{L^{2}\left(\rho\right)}$.
By Rademacher's theorem, $q'$ exists a.e. and 
\[
\left|q'\left(s\right)\right|\leq L_{K}\sigma_{v}\left\Vert c\right\Vert _{L^{2}\left(\rho\right)}\quad a.e.\;s.
\]
Taking $\eta\rightarrow0$ yields the stated bound with $\left\Vert g\right\Vert _{\mathcal{H}^{\left(1\right)}}$.
\end{proof}
For the fixed unit vector $v\in\mathbb{R}^{d}$, define the mixed
sampling measure 
\[
\rho=\frac{1}{2}\rho_{uni}+\frac{1}{2}\rho_{spec},
\]
where $\rho_{uni}$ is any probability measure with $\int\left\Vert a\right\Vert ^{2}d\rho_{uni}<\infty$
(as in \eqref{eq:d2}), and $\rho_{spec}$ is atomic along the direction
$v$: 
\[
\rho_{spec}=w_{B}\delta_{\left(v,B,t_{+}\right)}+w_{0}\delta_{\left(v,-\beta_{0},t_{+}\right)}+w_{1/2}\delta_{\left(v,-\beta_{1/2},t_{+}\right)}+w_{1}\delta_{\left(v,-\beta_{1},t_{+}\right)},
\]
with $B>0$, $w_{B},w_{0},w_{1/2},w_{1}\geq0$, $w_{B}+w_{0}+w_{1/2}+w_{1}=1$,
and $\beta_{0},\beta_{1/2},\beta_{1}$ to be determined below. 

Set 
\begin{align*}
G\left(s\right) & \coloneqq K^{\left(1\right)}\left(v,sv\right)\\
 & =\int_{Z}K\left(\left\langle a,v\right\rangle +b,t\right)K\left(\left\langle a,sv\right\rangle +b,t\right)d\rho\left(a,b,t\right),\quad s\in\left[0,1\right].
\end{align*}

\begin{lem}
\label{lem:4-2}Under assumption (A3), there exists $m_{*}>0$ such
that 
\[
G'\left(s\right)\geq m_{*}
\]
for a.e. $s\in\left[0,1\right]$, provided a constant $B$ is chosen
so that $s+B$, $1+B\in\left(u_{0},u_{1}\right)$ for all $s\in\left[0,1\right]$.
In particular, $G$ is strictly increasing and 
\[
\triangle\coloneqq G\left(1\right)-G\left(0\right)\geq m_{*}>0.
\]
\end{lem}

\begin{proof}
By dominated convergence (bounded $K$) and Rademacher, $G$ is a.e.
differentiable with 
\[
G'\left(s\right)=\int_{Z}K\left(\left\langle a,v\right\rangle +b,t\right)\partial_{1}K\left(\left\langle a,sv\right\rangle +b,t\right)\left\langle a,v\right\rangle d\rho\left(a,b,t\right)
\]
for a.e. $s\in\left[0,1\right]$. Split $\rho=\frac{1}{2}\rho_{uni}+\frac{1}{2}\rho_{spec}$.
Then 
\[
G'\left(s\right)=I_{uni}\left(s\right)+I_{spec}\left(s\right),
\]
with the uniform part
\[
I_{uni}\left(s\right)=\frac{1}{2}\int_{Z}K\left(\left\langle a,v\right\rangle +b,t\right)\partial_{1}K\left(\left\langle a,sv\right\rangle +b,t\right)\left\langle a,v\right\rangle d\rho_{uni}\left(a,b,t\right).
\]
Using boundedness $\left|K\right|\leq\left\Vert K\right\Vert _{\infty}$,
$\left|\partial_{1}K\right|\leq L_{K}$, and Cauchy-Schwarz, 
\[
I_{uni}\left(s\right)\geq-\frac{1}{2}\left\Vert K\right\Vert _{\infty}L_{K}\left(\int\left|\left\langle a,v\right\rangle \right|^{2}d\rho_{uni}\right)^{1/2}\eqqcolon-M,
\]
a finite constant (independent of $s$).

For the atomic part $I_{spec}$, the atom $\left(a,b,t\right)=\left(v,B,t_{+}\right)$
contributes positively to the derivative uniformly in $s$, because
$\left\langle v,v\right\rangle =1$ and by the choice of $B$ we have
\begin{align*}
\left(\left\langle v,v\right\rangle +B,t_{+}\right) & =\left(1+B,t_{+}\right)\in\left(u_{0},u_{1}\right)\times\left\{ t_{+}\right\} ,\\
\left(\left\langle v,sv\right\rangle +B,t_{+}\right) & =\left(s+B,t_{+}\right)\in\left(u_{0},u_{1}\right)\times\left\{ t_{+}\right\} 
\end{align*}
for all $s\in\left[0,1\right]$. Then by (A3), 
\[
K\left(1+B,t_{+}\right)\geq c_{1},\quad\partial_{1}K\left(1+B,t_{+}\right)\geq c_{2},\quad a.e.\:s\in\left[0,1\right].
\]

For the three auxiliary atoms, by boundedness and the Lipschitz assumption,
\[
\left|K\left(1-\beta,t_{+}\right)\partial_{1}K\left(1-\beta,t_{+}\right)\right|\leq\left\Vert K\right\Vert _{\infty}L_{K},\quad a.e.\:s\in\left[0,1\right].
\]
Altogether, 
\[
I_{spec}\geq\frac{1}{2}\left[w_{B}c_{1}c_{2}-\left(w_{0}+w_{1/2}+w_{1}\right)\left\Vert K\right\Vert _{\infty}L_{K}\right],
\]
and so 
\[
G'\left(s\right)\geq\frac{1}{2}\left[w_{B}c_{1}c_{2}-\left(w_{0}+w_{1/2}+w_{1}\right)\left\Vert K\right\Vert _{\infty}L_{K}\right]-M\eqqcolon m_{*}
\]
for a.e. $s\in\left[0,1\right]$. 

We can choose the weights $w_{B},w_{0},w_{1/2},w_{1}$ and $\rho_{uni}$
so that $m_{*}>0$. Hence $G$ is strictly increasing and $\triangle=G\left(1\right)-G\left(0\right)\geq m_{*}>0$. 
\end{proof}
Let $\beta_{0}=G\left(0\right)$, $\beta_{1}=G\left(1\right)$, and
$\triangle=\beta_{1}-\beta_{0}>0$. Assume (possibly after shrinking
the slice length and reparametrizing $s$ to $\left[0,1\right]$)
that 
\[
\triangle\leq u_{1}-u_{0}.
\]
Choose a shift $b_{out}$ so that 
\begin{equation}
\left[\beta_{0}+b_{out},\beta_{1}+b_{out}\right]\subset\left(u_{0},u_{1}\right),\label{eq:d-4}
\end{equation}
equivalently, 
\[
b_{out}\in\left(u_{0}-\beta_{0},u_{1}-\beta_{1}\right).
\]
(Thus, whenever $u\in\left[\beta_{0},\beta_{1}\right]$, we have $u+b_{out}\in\left(u_{0},u_{1}\right)$).

Augment $\rho_{spec}$ by including an additional atom at 
\[
z_{out}=\left(a,b,t\right)=\left(v,b_{out},t_{+}\right)
\]
with weight $w_{out}>0$. We renormalize the other $w$-weights so
that $\rho_{spec}$ remains a probability measure. The lower bound
$m_{*}$ from \eqref{lem:4-2} is preserved by taking $w_{out}$ small
and (if needed) slightly increasing $w_{B}$. 

Now define 
\begin{equation}
f_{2}\left(x\right)\coloneqq\alpha K\left(K^{\left(1\right)}\left(v,x\right)+b_{out},t_{+}\right)\label{eq:d5}
\end{equation}
with $\alpha>0$, a scalar to be chosen later. 
\begin{lem}
\label{lem:d3}With the atom $z_{out}$ present in $\rho_{spec}$,
one has $f_{2}\in\mathcal{H}^{\left(2\right)}$ and 
\[
\left\Vert f_{2}\right\Vert _{\mathcal{H}^{\left(2\right)}}\leq\alpha\sqrt{\frac{2}{w_{out}}}.
\]
\end{lem}

\begin{proof}
In $\mathcal{H}^{\left(2\right)}$, functions are of the form 
\[
x\mapsto\int_{Z}c\left(z\right)K\left(K^{\left(1\right)}\left(a,x\right)+b,t\right)d\rho\left(z\right).
\]
Take $c$ supported at $z_{out}$ with 
\[
c\left(z_{out}\right)=\frac{2\alpha}{w_{out}},\quad c\equiv0\;\text{elsewhere}.
\]
Since the mass of $z_{out}$ inside $\rho$ is $\frac{1}{2}w_{out}$,
\[
\int_{Z}c\left(z\right)\psi^{\left(2\right)}_{z}d\rho\left(z\right)=\frac{1}{2}w_{out}\cdot\frac{2\alpha}{w_{out}}K\left(K^{\left(1\right)}\left(v,x\right)+b_{out},t_{+}\right)=f_{2}\left(x\right),
\]
so $f_{2}\in\mathcal{H}^{\left(2\right)}$. Moreover, 
\[
\left\Vert f_{2}\right\Vert _{\mathcal{H}^{\left(2\right)}}\leq\left\Vert c\right\Vert _{L^{2}\left(\rho\right)}=\left(\frac{1}{2}w_{out}\left(\frac{2\alpha}{w_{out}}\right)^{2}\right)^{1/2}=\alpha\sqrt{\frac{2}{w_{out}}}.
\]
\end{proof}
\begin{lem}
\label{lem:d-4}Let $h\left(s\right)=f_{2}\left(sv\right)$, $s\in\left[0,1\right]$.
Then, for a.e. $s\in\left[0,1\right]$, 
\[
\left|h'\left(s\right)\right|\geq\alpha c_{2}m_{*}>0.
\]
\end{lem}

\begin{proof}
Write $u\left(s\right)=G\left(s\right)=K^{\left(1\right)}\left(v,sv\right)\in\left[\beta_{0},\beta_{1}\right]$
(since $G$ is strictly increasing). By \eqref{eq:d-4}, 
\[
u\left(s\right)+b_{out}\in\left(u_{0},u_{1}\right)
\]
for all $s$. By the chain rule (valid a.e. by Rademacher and dominated
convergence), 
\[
h'\left(s\right)=\alpha\partial_{1}K\left(u\left(s\right)+b_{out},t_{+}\right)\cdot G'\left(s\right).
\]
Assumption (A3) gives $\partial_{1}K\left(\cdot,t_{+}\right)\geq c_{2}$
on $\left(u_{0},u_{1}\right)$, and \eqref{lem:4-2} gives $G'\left(s\right)\geq m_{*}$
a.e. Hence $\left|h'\left(s\right)\right|\geq\alpha c_{2}m_{*}>0$
a.e.
\end{proof}
Let $h\left(s\right)=f_{2}\left(sv\right)$ for $s\in\left[0,1\right]$.
From \prettyref{lem:d-4} we have the uniform slope 
\begin{equation}
\left|h'\left(s\right)\right|\geq S\;a.e.\:\text{on }\left[0,1\right],\quad S\coloneqq\alpha c_{2}m_{*}>0.\label{eq:d-5}
\end{equation}
For any $g\in\mathcal{H}^{\left(1\right)}$, write $q\left(s\right)=g\left(sv\right)$.
By \prettyref{lem:d-1}, 
\begin{equation}
\left|q'\left(s\right)\right|\leq L\:a.e.\:\text{on }\left[0,1\right],\quad L\coloneqq L_{K}\sigma_{v}\left\Vert g\right\Vert _{\mathcal{H}^{\left(1\right)}}.\label{eq:d-6}
\end{equation}
Define the error $r\left(s\right)\coloneqq q\left(s\right)-h\left(s\right)$.
Then 
\[
r\left(1\right)-r\left(0\right)=\left(q\left(1\right)-q\left(0\right)\right)-\left(h\left(1\right)-h\left(0\right)\right),
\]
and by \eqref{eq:d-5}--\eqref{eq:d-6} we have 
\begin{align*}
\left|q\left(1\right)-q\left(0\right)\right| & \leq\int^{1}_{0}\left|q'\left(s\right)\right|ds\leq L,\\
\left|h\left(1\right)-h\left(0\right)\right| & \geq\int^{1}_{0}\left|h'\left(s\right)\right|ds\geq S.
\end{align*}
Hence 
\begin{equation}
\left|r\left(1\right)-r\left(0\right)\right|\geq S-L.\label{eq:d-7}
\end{equation}
Note that one can tune the construction constants $m_{*}$ and $\sigma_{v}$
to ensure $S-L>0$. 

We now relate the endpoint gap of $r$ to its $L^{2}$-error after
optimal vertical alignment. 
\begin{lem}
\label{lem:d-5}For any absolutely continuous $r\colon\left[0,1\right]\rightarrow\mathbb{R}$,
\[
\min_{c\in\mathbb{R}}\int^{1}_{0}\left(r\left(s\right)-c\right)^{2}ds\geq\frac{\left(r\left(1\right)-r\left(0\right)\right)^{2}}{12}.
\]
\end{lem}

\begin{proof}
(sketch) Among all $r$ with prescribed endpoints $r\left(0\right),r\left(1\right)$,
the function $J\left(r\right)=\min_{c}\int\left(r-c\right)^{2}$ is
minimized by the affine function connecting the endpoints. For the
affine $r\left(s\right)=r\left(0\right)+Ds$ (with $D=r\left(1\right)-r\left(0\right)$),
the minimizer $c$ is the mean $r\left(0\right)+D/2$, giving $J\left(r\right)=\int^{1}_{0}\left(Ds-D/2\right)^{2}ds=D^{2}/12$.
\end{proof}
\begin{lem}
If a shallow $g\in\mathcal{H}^{\left(1\right)}$ achieves slice error
$\int^{1}_{0}\left|q-h\right|^{2}ds\leq\epsilon^{2}$, then
\begin{equation}
\left\Vert g\right\Vert _{\mathcal{H}^{\left(1\right)}}\ge\frac{\left(\alpha c_{2}m_{*}-\sqrt{12}\epsilon\right)_{+}}{L_{K}\sigma_{v}}.\label{eq:d-8}
\end{equation}
\end{lem}

\begin{proof}
Applying \prettyref{lem:d-5} to $r\left(s\right)=q\left(s\right)-h\left(s\right)$
and using \eqref{eq:d-7}, 
\[
\min_{c}\int^{1}_{0}\left(q\left(s\right)-h\left(s\right)-c\right)^{2}ds\geq\frac{\left(S-L\right)^{2}}{12}.
\]
Since $\int^{1}_{0}\left|q-h\right|^{2}\geq\min_{c}\int^{1}_{0}\left(q-r-c\right)^{2}$,
we obtain the slope-norm tradeoff 
\begin{equation}
\int^{1}_{0}\left|q-h\right|^{2}ds\geq\frac{1}{12}\left(S-L\right)^{2}.\label{eq:d-9}
\end{equation}
Consequently, if $\int^{1}_{0}\left|q-h\right|^{2}ds\leq\epsilon^{2}$,
then combining \eqref{eq:d-6} and \eqref{eq:d-9} gives 
\[
L_{K}\sigma_{v}\left\Vert g\right\Vert _{\mathcal{H}^{\left(1\right)}}\geq S-\sqrt{12}\epsilon
\]
and \eqref{eq:d-8} follows.

Below is the general $K$ depth-separation on the slice, analogous
to the ReLU case in \prettyref{thm:7}, with explicit dependence on
the construction parameters.
\end{proof}
\begin{thm}[Depth Separation for General $K$]
\label{thm:d-7}Let the approximation error be measured on the slice
$x=sv$, $s\in\left[0,1\right]$, with respect to the uniform measure
$\mu_{v}$. Given $\epsilon>0$ sufficiently small, there exists a
function $f_{2}\in\mathcal{H}^{\left(2\right)}$, such that 
\[
\frac{N_{shallow}\left(\epsilon\right)}{N_{deep}\left(\epsilon\right)}\sim m^{2}_{*}\rightarrow\infty
\]
where $m_{*}$ is a tunable construction parameter.
\end{thm}

\begin{proof}
Fix the slice $x=sv$, $s\in\left[0,1\right]$, with slice measure
$\mu_{v}$ uniform on $\left[0,1\right]$. Recall that, by \prettyref{lem:c7},
for any $f\in\mathcal{H}^{\left(l\right)}$ of the form 
\[
f\left(x\right)=\int_{Z}c\left(z\right)\psi^{\left(l\right)}_{z}\left(x\right)d\rho\left(z\right),
\]
where $\psi^{\left(1\right)}_{z}\left(x\right)=K\left(\left\langle a,x\right\rangle +b,t\right)$,
and $\psi^{\left(l\right)}_{z}\left(x\right)=K\left(K^{\left(l-1\right)}\left(a,x\right)+b,t\right)$,
define the MC estimator 
\[
\hat{f}_{N}\left(x\right)=\frac{1}{N}\sum^{N}_{j=1}c\left(z_{j}\right)\psi^{\left(l\right)}_{z_{j}}\left(x\right)
\]
with $z_{1},\dots,z_{N}\overset{i.i.d}{\sim}\rho$. Then $\mathbb{E}[\hat{f}_{N}]=f$,
and 
\begin{equation}
\mathbb{E}\left[\left\Vert \hat{f}_{N}-f\right\Vert ^{2}_{L^{2}\left(\mu_{v}\right)}\right]\leq\frac{V_{l}}{N}\left\Vert c\right\Vert ^{2}_{L^{2}\left(\rho\right)},\quad V_{l}\coloneqq\sup_{z}\int^{1}_{0}\left|\psi^{\left(l\right)}_{z}\left(sv\right)\right|^{2}ds.\label{eq:d-11}
\end{equation}
Under (A1)--(A2) with $K$ bounded, we have $V_{l}\leq\left\Vert K\right\Vert ^{2}_{\infty}$
for every $l$ (since $|\psi^{\left(l\right)}_{z}|\leq\left\Vert K\right\Vert _{\infty}$).

For the target function (see \eqref{eq:d5})
\[
f_{2}\left(x\right)=\alpha K\left(K^{\left(1\right)}\left(v,x\right)+b_{out},t_{+}\right),
\]
take the coefficient $c$ supported at the atom $z_{out}=\left(v,b_{out},t_{+}\right)$
with value $c\left(z_{out}\right)=2\alpha/w_{out}$ and zero elsewhere.
As computed before (\prettyref{lem:d3}), 
\[
\left\Vert c\right\Vert ^{2}_{L^{2}\left(\rho\right)}=\frac{2\alpha^{2}}{w_{out}}.
\]
Applying \eqref{eq:d-11} with $l=2$, 
\[
\mathbb{E}\left[\left\Vert \hat{f}_{N}-f_{2}\right\Vert ^{2}_{L^{2}\left(\mu_{v}\right)}\right]\leq\frac{V_{2}}{N}\cdot\frac{2\alpha^{2}}{w_{out}}.
\]
Thus, to achieve squared slice error $\leq\epsilon^{2}$, we get
\[
N\geq\frac{2V_{2}}{w_{out}}\frac{\alpha^{2}}{\epsilon^{2}}.
\]
Thus, the minimal number of samples satisfies
\begin{equation}
N_{\text{deep}}\left(\epsilon\right)\leq\frac{2V_{2}}{w_{out}}\frac{\alpha^{2}}{\epsilon^{2}}.\label{eq:d12}
\end{equation}

Let $g\in\mathcal{H}^{\left(1\right)}$ be any shallow candidate and
$\hat{g}_{N}$ its MC estimator. By \eqref{eq:d-11} with $l=1$,
\begin{equation}
\mathbb{E}\left[\left\Vert \hat{g}_{N}-g\right\Vert ^{2}_{L^{2}\left(\mu_{v}\right)}\right]\leq\frac{V_{1}}{N}\left\Vert g\right\Vert ^{2}_{\mathcal{H}^{\left(1\right)}}.\label{eq:d-12}
\end{equation}
From \eqref{eq:d-8}, if $\left\Vert g-f_{2}\right\Vert ^{2}_{L^{2}\left(\mu_{v}\right)}\leq\epsilon^{2}/2$,
then 
\begin{equation}
\left\Vert g\right\Vert _{\mathcal{H}^{\left(1\right)}}\geq\frac{\left(\alpha c_{2}m_{*}-\sqrt{6}\epsilon\right)_{+}}{L_{K}\sigma_{v}}.\label{eq:d-13}
\end{equation}
To make the total expected error satisfy $\mathbb{E}\left[\left\Vert \hat{g}_{N}-f_{2}\right\Vert ^{2}_{L^{2}\left(\mu_{v}\right)}\right]\leq\epsilon^{2}$,
it suffices to require the two terms to be $\leq\epsilon^{2}/2$ each:
\[
\mathbb{E}\left[\left\Vert \hat{g}_{N}-g\right\Vert ^{2}_{L^{2}\left(\mu_{v}\right)}\right]\leq\epsilon^{2}/2,\quad\left\Vert g-f_{2}\right\Vert ^{2}_{L^{2}\left(\mu_{v}\right)}\leq\epsilon^{2}/2.
\]
Using \eqref{eq:d-12}-\eqref{eq:d-13}, any such shallow procedure
must have 
\[
N_{\text{shallow}}\left(\epsilon\right)\geq\frac{2V_{1}}{\epsilon^{2}}\left(\frac{\alpha c_{2}m_{*}-\sqrt{6}\epsilon}{L_{K}\sigma_{v}}\right)^{2}_{+}.
\]
In particular, for sufficiently small $\epsilon$, 
\begin{equation}
N_{\text{shallow}}\left(\epsilon\right)\apprge\frac{2V_{1}}{\epsilon^{2}}\frac{\alpha^{2}c^{2}_{2}m^{2}_{*}}{L^{2}_{K}\sigma^{2}_{v}}.\label{eq:d-14}
\end{equation}
Combining \eqref{eq:d12} and \eqref{eq:d-14}, the ratio satisfies,
for small $\epsilon$, 
\[
\frac{N_{\text{shallow}}\left(\epsilon\right)}{N_{\text{deep}}\left(\epsilon\right)}\apprge\frac{V_{1}}{V_{2}}\frac{w_{out}}{L^{2}_{K}}\frac{c^{2}_{2}}{\sigma^{2}_{v}}m^{2}_{*}
\]

All constants $V_{1}$, $V_{2}$, $w_{out}$, $c_{2}$, $L_{K}$,
$\sigma_{v}$ are fixed by the construction. By \prettyref{lem:4-2},
we can make $m_{*}$ arbitrarily large (by increasing $w_{B}$ and
adjusting $\rho_{uni}$ as discussed). Hence the ratio can be made
arbitrarily large: 
\[
\frac{N_{\text{shallow}}\left(\epsilon\right)}{N_{\text{deep}}\left(\epsilon\right)}\rightarrow\infty,\;\text{as }m_{*}\rightarrow\infty.
\]
\end{proof}

\bibliographystyle{plain}
\bibliography{ref}

\end{document}